\documentclass[a4paper,11pt]{article}

\usepackage{hyperref}
\usepackage{natbib}
\usepackage{amsmath}
\usepackage{amssymb}
\usepackage{amsthm}
\usepackage[dvips]{graphicx}
\usepackage{caption}
\usepackage{subcaption}
\usepackage{color}
\usepackage[all,dvips,arc,curve,color,frame]{xy}
\usepackage[normalem]{ulem}  
\usepackage{cancel}  

\def\ci{\perp\!\!\!\perp}
\newcommand{\iend}[1]{\; \mathrm{d}#1}

\newtheorem{lemma}{Lemma}
\newtheorem{theorem}{Theorem}
\newtheorem{remark}{Remark}

\DeclareMathOperator*{\argmax}{arg\,max}

\title{Score-based Causal Learning in Additive~Noise~Models}
\author{Christopher Nowzohour\footnote{Corresponding Author}\\Seminar f\"ur Statistik, ETH Z\"urich\\nowzohour@stat.math.ethz.ch \and Peter B\"uhlmann\\Seminar f\"ur Statistik, ETH Z\"urich\\buhlmann@stat.math.ethz.ch}

\begin{document}

\bibliographystyle{plainnat}

\maketitle

\begin{abstract}
Given data sampled from a number of variables, one is often interested in the underlying causal relationships in the form of a directed acyclic graph. In the general case, without interventions on some of the variables it is only possible to identify the graph up to its Markov equivalence class. However, in some situations one can find the true causal graph just from observational data, for example in structural equation models with additive noise and nonlinear edge functions. Most current methods for achieving this rely on nonparametric independence tests. One of the problems there is that the null hypothesis is independence, which is what one would like to get evidence for. We take a different approach in our work by using a penalized likelihood as a score for model selection. This is practically feasible in many settings and has the advantage of yielding a natural ranking of the candidate models. When making smoothness assumptions on the probability density space, we prove consistency of the penalized maximum likelihood estimator. We also present empirical results for simulated scenarios and real two-dimensional data sets (cause-effect pairs) where we obtain similar results as other state-of-the-art methods.
\end{abstract}

\section{Introduction}
Statistical causal inference is an important but relatively new field. Traditionally, most statistical statements and assertions are associational ($X$ and $Y$ are correlated), rather than causal (changes in $X$ cause changes in $Y$). While the former are statements about the joint distribution, the latter are about the underlying causal mechanisms. In practice, the relevant question often is whether variable $X$ has a causal effect\footnote{$X$ has a causal effect on $Y$ if manipulating $X$ changes the distribution of $Y$, see \cite{PeaJ00}.} on variable $Y$, possibly mediated by some other variables $Z_1, \ldots, Z_d$ in the causal network. In general, the only way to completely identify the causal model is by performing experiments (interventions). However, it is often possible to at least narrow down the space of candidate models by using only observational data \citep{VerTP91, SpiPA93}. There are many situations where one is dependent on purely observational data---either because performing experiments is infeasible (e.g.\ astronomical data), unethical (e.g.\ clinical cancer studies), or both (e.g.\ economical data). Some real-life examples include identifying gene expression networks \citep{StaAA12, SteDA12} and analysing fMRI data from the human brain \citep{RamJA10}.

When modeling causal networks between some given variables, structural equation models are used frequently, where each variable is expressed as a function of some other variables (its causes) as well as some noise. Thus the model is determined by the cause-effect structure (in the form of a directed graph over the variables), the functional dependencies, and the joint distribution of the noise terms. Assumptions typically made include that the underlying causal model is acyclic (i.e.\ there are no feedback loops) and that the noise terms are independent (i.e.\ there are no unobserved variables). We furthermore assume that the noise is additive, i.e.\ the effect variable minus some noise term is a deterministic function of the cause variables. Although quite restrictive, this is a common assumption in many other settings (e.g.\ regression) and allows straightforward estimation. The standard case then is to parameterize the model by making the functional dependencies linear and the noise Gaussian\footnote{In fact, this is how structural equation models where first introduced and continue to be used today \citep{BolKA89}.}. In this case the space of candidate models (in the form of directed acyclic graphs) clusters in equivalence classes, which prohibit full identification---every model in a given equivalence class can induce the same joint distribution over the variables. In a sense, this is quite exceptional, however. It has been shown that as soon as one departs from the linearity or the Gaussianity assumptions the model becomes fully identifiable\footnote{Except for a set of degenerate cases of measure zero.} \citep{ShiSA06, HoyPA09, ZhaKH09, PetJ11, PetJB12}. We are thus interested in the nonparametric case, where either the functional dependencies are nonlinear or the noise terms are non-Gaussian (or both). An inference procedure for this case based on nonparametric independence tests has been suggested by \citet{MooJA09}. Their method is using the fact that when fitting the wrong model the noise terms will not be independent. There are a few problems with this approach, however. First, the null hypothesis of the tests employed is independence, which is what one would like to show, and statistical hypothesis testing only allows to reject such hypotheses. Second, because of the many tests involved there is a multiple testing problem. Third, nonparametric independence testing among many variables is statistically hard, and the tests tend to be computationally intensive.

We take a different approach in the form of a score-based method, which is consistent, fast, and easily adaptable to greedy methods for large problems. Score-based methods are widely used for fitting Gaussian structural equation models \citep{ChiDB02} or discrete Bayesian networks \citep{KolDF09}. Maximum a posteriori estimation was used in the setting of non-linear models with Gaussian noise by \citet{ImoSA02}. Two other score-based methods have recently been proposed: for the parametric setting of Gaussian and linear models with same error variances \citep{PetJB12} and for linear models with non-Gaussian noise \citep{HyvAS13}. Most closely related to this paper is an approach from \citet{BuePA13}. They consider a semi-parametric structural equation model with additive, nonlinear functions in the parental variables and additive Gaussian noise, and they prove consistency and present an algorithm for cases with potentially many variables. In contrast, we consider here a model with a nonparametric specification of the error distribution (while the focus is on cases with few variables only). Thus, our model is more general but harder to estimate from data. We propose a penalized maximum likelihood method and prove its asymptotic consistency for finding the true underlying graph provided some technical assumptions about the class of probability densities hold. Our nonparametric setting also includes the well-known LiNGAM model \citep{ShiSA06} as a special case, and thus we provide here a score-based approach for LiNGAM. Independent work by \citet{KpoSA13} considers a similar problem as ours: however, while they only treat the case with two variables, we allow for more realistic multivariate settings.

This paper is organized as follows: In Section~2 we review the basic notation and definitions we will use later on before describing our method. In Section~3 we present our main theorem and the assumptions for proving consistency in the large sample limit. In Section~4 we discuss simulation results showing that the method works in practice under controlled conditions. In Section~5, we test our method on some real-world datasets and compare it to other causal inference methods.

\section{The Method}
Suppose data is sampled from real-valued random variables $X_1, \ldots, X_d$, which have some causal structure. We are interested in finding this causal structure (in the form of a directed acyclic graph) just by using observational data. Before we describe our method and the assumptions it rests on, we will give definitions of some of the basic terms used in this paper (some of which can be found in e.g.\ \citet{LauSL96}, \citet{PeaJ00}, \citet{TriH83}).

\subsection{Notation and Definitions}
Given a set of vertices $\mathcal{V}=\{ 1, \ldots, d \}$ and edges $\mathcal{E} \subset \mathcal{V} \times \mathcal{V}$, we define the $d$-dimensional \emph{graph} $G$ as the ordered pair $(\mathcal{V}, \mathcal{E})$. If $\mathcal{E}$ is asymmetric, $G$ is called a \emph{directed graph}. Given two vertices $\alpha, \beta \in \mathcal{V}$, a \emph{directed path} of length $n$ from $\alpha$ to $\beta$ is a sequence of vertices $\alpha=v_0, \ldots, v_n=\beta$, s.t.\ $(v_i, v_{i+1}) \in \mathcal{E} ~ \forall i=0,\ldots,n-1$. If $G$ is directed and for all $v \in \mathcal{V}$ there is no path of length $n \ge 1$ from $v$ to itself, then $G$ is called a \emph{directed acyclic graph (DAG)}. If $\mathcal{V}' \subseteq \mathcal{V}$ and $\mathcal{E}' \subseteq \mathcal{E}|_{\mathcal{V}' \times \mathcal{V}'}$, then $G' = (\mathcal{V}', \mathcal{E}')$ is called a \emph{subgraph} of $G$, and we write $G' \subseteq G$. If $\mathcal{E}' \subset \mathcal{E}|_{\mathcal{V}' \times \mathcal{V}'}$, we call $G'$ a \emph{proper} subgraph of $G$ and write $G' \subset G$. In a graph $G$ we define the \emph{parents} of a vertex $v$ as the set $\mathrm{pa}_G (v) := \{ u \in \mathcal{V}: (u, v) \in \mathcal{E} \}$. The structural Hamming distance (SHD) between two graphs $G, G'$ is defined as the number of single edge operations (edge additions, deletions, reversals) necessary to transform $G$ into $G'$.

A joint density $p$ over $X_1, \ldots, X_d$ is \emph{Markov} with respect to a DAG $D$, if it factorizes along $D$:
\begin{align}
  \label{cond:markov}
  p(x_1, \ldots, x_d) = \prod_{k=1}^d p \left( x_k | \{ x_l \}_{l \in \mathrm{pa}_D(k)} \right).
\end{align}
A DAG $D$ is \emph{causally minimal} with respect to a joint density $p$, if $\nexists D' \subset D$ s.t.\ $p$ is Markov with respect to $D'$.

A \emph{structural equation model (SEM)} $M = \{f_k, p_{\epsilon_k}\}_{k=1,\ldots,d}$ is a set of functions $f_k$ and densities $p_{\epsilon_k}$, specifying each variable $X_k$ as a function of some of the other variables and a noise term $\epsilon_k$ (independent of the other noise terms) with density $p_{\epsilon_k}$. The model $M$ induces a DAG $D$, where a directed edge $(k, l)$ is added if the function for $X_l$ directly depends on $X_k$. We will assume in this paper, that $M$ is \emph{recursive}, i.e. its graph $D$ is actually a DAG. We can write the model equations as
\begin{align*}
  X_k = f_k( \{ X_l \}_{l \in \mathrm{pa}_D(k)}), \epsilon_k), \qquad k=1, \ldots d.
\end{align*}
If the functions are additive in the noise, i.e.\ if
\begin{align}
  \label{eq:anm}
  X_k = f_k( \{ X_l \}_{l \in \mathrm{pa}_D(k)}) ) + \epsilon_k, \qquad k=1, \ldots, d,
\end{align}
the model is called an \emph{additive noise model (ANM)}. We call $\mathcal{M} := (\mathcal{F}, \mathcal{P}^\epsilon)$ a \emph{functional model class}\footnote{Here we implicitly assume that the model has additive noise.} of dimension $d$ if $\mathcal{F} \subset C^0 (\mathbb{R}^{d-1})$ is a class of functions containing the possible edge functions $f_k$ and $\mathcal{P}^\epsilon$ is a class of univariate probability densities containing the possible error densities $p_{\epsilon_k}$.

The joint density of an ANM is of the form \eqref{cond:markov} and thus Markov to its DAG $D$. Vice versa we say that $D$ induces a class of joint densities $\mathcal{P}$ on $X_1, \ldots, X_d$ from a functional model class $\mathcal{M}$, where
\begin{align}
  \label{defn:induced}
  \mathcal{P} = \left\{ \prod_{k=1}^d p_k \left( x_k - f_k ( \{ x_l \}_{l \in \mathrm{pa}_D(k)}) ) \right) : f_k \in \mathcal{F}, p_k \in \mathcal{P}^\epsilon \right\}.
\end{align}
Thus $\mathcal{P}$ contains all joint densities that can be generated by ANMs from class $\mathcal{M}$ with DAG $D$. The class $\mathcal{M}$ is said to be \emph{identifiable}, if the intersection of any two density classes $\mathcal{P}^1, \mathcal{P}^2$ induced by distinct graphs $D_1, D_2$ only contains densities for which there exists a unique graph that is causally minimal. We assume throughout the paper that the data generating process is an ANM with associated causally minimal DAG $D_0$ with induced density class $\mathcal{P}^0$ and true joint density $p^0 \in \mathcal{P}^0$. Causal minimality here essentially means that every edge in $D$ creates a dependency in the joint distribution (i.e.\ there is an edge from $X_l$ to $X_k$ only if $f_k$ is not constant in $x_l$).

For the density class, we often consider the weighted \emph{Sobolev space} of functions $W_r^s(\mathbb{R}^n, \langle \cdot \rangle^\beta)$ which is defined as follows:
\begin{align*}
  W_r^s (\mathbb{R}^n, \langle \cdot \rangle^\beta) := \left\{ f \in L^r (\mathbb{R}^n) : D^\alpha (f\cdot \langle \cdot \rangle^\beta) \in L^r (\mathbb{R}^n) ~ \forall |\alpha | \le s \right\},
\end{align*}
where $\langle x \rangle^\beta = (1+\|x\|^2)^{\beta/2}$ is a polynomial weighting function parametrized by $\beta \in \mathbb{R}$, $D^\alpha$ is the partial derivative operator according to the multi-index $\alpha$, and $r$, $s$ are integers at least 1. Note that for $\beta=0$ this is the usual Sobolev space, while for $\beta>0$ this is more restrictive (as the tails get bigger weights), and for $\beta < 0$ it is less restrictive. We will mostly be interested in the $\beta < 0$ case.

\sloppy
\subsection{Penalized maximum likelihood estimation}

\label{sec:method}
We now describe our method to learn the true causal structure from data. Suppose we measure $d$ variables, and we have $n$ i.i.d.\ samples $\{ x_k^j \}$ with $j=1, \ldots, n$ and $k=1, \ldots, d$. Let $D_1, \ldots, D_{N}$ be the candidate DAGs under consideration\footnote{E.g.\ all DAGs with $d$ nodes.} and $\mathcal{P}^1, \ldots, \mathcal{P}^{N}$ their induced density classes for some model class $\mathcal{M}$. If $\mathcal{M}$ is identifiable, we aim to infer the true DAG $D_0$ by finding the density class $\mathcal{P}^0$ that contains the true joint density $p^0$ (if there is more than one such class, we choose the one corresponding to the smallest graph). Of course, we do not know $p^0$---instead we estimate it by computing ``best representatives'' $\hat{p}_n^i$ from each class $\mathcal{P}^i$. These are chosen via nonparametric maximum likelihood:
\begin{align*}
  \hat{p}_n^i = \argmax_{p \in \mathcal{P}^i} \sum_{j=1}^n \log p (x_1^j, \ldots, x_d^j).
\end{align*}
Then, each model is scored with a penalized log-likelihood:
\begin{align} \label{def:score}
   S_n^i = \frac{1}{n} \sum_{j=1}^n \log \hat{p}_n^i (x_1^j, \ldots, x_d^j) - \#(\mathrm{edges})_i \cdot a_n,
\end{align}
where $a_n$ controls the strength of the penalty. Taking the maximum over these scores we get the estimator
\begin{align*}
  \hat{D}_n = D_{\hat{I}_n}, \quad \mathrm{where} \quad \hat{I}_n = \argmax_{i=1,\ldots,N} S_n^i.
\end{align*}
Hence the estimated DAG is $D_{\hat{I}}$. We will show in Section~\ref{chapter:theory} that this procedure is consistent for $a_n$ proportional to $1/\log n$ and that therefore $\hat{D}_n=D_0$ in the large sample limit.

The question arises how to find the maximum likelihood estimators $\hat{p}_n^i$ in each class in this nonparametric setting. We present here an exemplary procedure that has proved useful in practice. To estimate the edge functions of the SEM, we employ a nonparametric regression method. The error densities are then inferred from the residuals using a density estimation method. The estimated joint density is finally given by the product of the residual densities, in accordance with \eqref{defn:induced}.

This gives the following three-step procedure for each DAG $D_i$:
\begin{enumerate}
  \item For each node $k$ estimate the residuals $\hat{\epsilon}_k$ by nonparametrically regressing $X_k$ on $\{ X_l \}_{l \in \mathrm{pa}_{D_i(k)}}$. If $\mathrm{pa}_{D_i(k)} = \varnothing$, set $\hat{\epsilon}_k = x_k$.
  \item For each node $k$ estimate the residual densities $\hat{p}_{\epsilon_k}$ from the estimated residuals $\hat{\epsilon}_k$.
  \item Compute the penalized likelihood score
    \begin{align*}
      S_n^i= \frac{1}{n} \sum_{j=1}^n \sum_{k=1}^d \log \hat{p}_{\epsilon_k} (\hat{\epsilon}^j_k) - \#(\mathrm{edges})_i \cdot a_n.
    \end{align*}
\end{enumerate}
Of course, an exhaustive search over all DAGs is only feasible for small values of $d$, since the number of DAGs grows super-exponentially with the number of vertices\footnote{The first few values of the number of DAGs $N(d)$ with $d$ nodes are $N(2)=3$, $N(3)=25$, $N(4)=543$, $N(5)=29281$, $N(6)=3781503$, for example.} and nonparametric regression in $d$ dimensions is ill-posed in general without making structural constraints, due to the curse of dimensionality\footnote{The latter problem can be dealt with in certain cases, e.g.\ additive models, where the edge functions are additive in the parental variables.}. The methods used in steps 1 and 2 should be chosen depending on the model class $\mathcal{M}$. Examples are (generalized) additive model regression (GAM) for step 1 and kernel density estimation for step 2.

As an illustration we look at the two-dimensional case, where there are only two variables $X_1$ and $X_2$. There are three DAGs inducing the following models:
\begin{align*}
  \begin{array}{ll}
    D_1: & X_1 \longrightarrow X_2 \\
    & X_1 = \epsilon_1 \\
    & X_2 = f(X_1) + \epsilon_2 \\
    & p_1 (x_1,x_2) = p_{X_1} (x) \cdot p_{X_2|X_1} (x_2|x_1) = p_{\epsilon_1} (x_1) \cdot p_{\epsilon_2} (x_2-f(x_1)) \\
    & \\
    D_2: & X_1 \longleftarrow X_2 \\
    & X_1 = g(X_2) + \epsilon_1 \\
    & X_2 = \epsilon_2 \\
    & p_2 (x_1,x_2) = p_{X_1|X_2} (x_1|x_2) \cdot p_{X_2} (x_2) = p_{\epsilon_1} (x_1-g(x_2)) \cdot p_{\epsilon_2} (x_2) \\
    & \\
    D_3: & X_1 \ci X_2 \\
    & X_1 = \epsilon_1 \\
    & X_2 = \epsilon_2 \\
    & p_3 (x_1,x_2) = p_{X_1} (x_1) \cdot p_{X_2} (x_2) = p_{\epsilon_1} (x_1) \cdot p_{\epsilon_2} (x_2)
  \end{array}
\end{align*}
We do steps 1, 2, and 3 as described above and choose the model with the highest (log-)likelihood penalized likelihood score.

Comparing this score-based approach with independence-test-based methods, the main difference occurs at step 2, where we estimate the residual densities instead of testing their independence. In terms of complexity, we swap one $d$-dimensional independence test againt $d$ univariate density estimations. Simulations show that this is faster by a factor on the order of 100 with current implementations. However, even though we do not test residual independence directly, it is still the discriminatory property by which to identify the true model. By constructing the densities according to \eqref{defn:induced}, we enforce the error terms to be independent in the estimated joint density. If they are not actually, the considered model will obtain a poor score. Thus, we are searching for the best fitting densities where the errors are independent.

\section{Theoretical Results}
\label{chapter:theory}
We now show that our method is consistent, i.e.\ that it will identify the true underlying DAG given enough samples. In the following $\mathcal{P}_D$ denotes the induced density class of DAG $D$. We make the following assumptions:
\begin{enumerate}
  \renewcommand{\theenumi}{(A\arabic{enumi})}
  \renewcommand{\labelenumi}{\theenumi}

  \item \label{ass:anm}
    Identifiability: The data $\{ x_k^j \}_{\substack{k=1, \ldots, d \\ j=1, \ldots, n}}$ are i.i.d.\ realizations (over $j = 1, \ldots, n$) of an identifiable structural equation model with induced $d$-dimensional DAG $D_0$. In particular, the SEM can be the additive noise model \eqref{eq:anm} with nonlinear edge functions $f_k$ or non-Gaussian noise variables\footnote{Excluding a set of exceptions of measure zero \citep[Theorem 1]{HoyPA09}.} $\epsilon_k$ for all $k = 1, \ldots, d$ \citep[Lemma 1]{PetJ11}. There are no hidden variables, i.e.\ the noise terms are jointly independent.

  \item \label{ass:minimality}
    Causal Minimality: There is no proper subgraph $D'$ of $D_0$, s.t.\ $p^0$ is Markov with respect to $D'$.

  \item \label{ass:sobolev}
    Smoothness of log-densities: For all DAGs $D$ the log-densities of $\mathcal{P}_D$ (restricted to their respective support) are elements of a bounded weighted Sobolev space. That is $\exists r \ge 1$, $s > d$, $\beta < 0$, $C > 0$ s.t.\
    \begin{align*}
      \sum_{|\alpha| \le s} \| D^\alpha ( \langle \cdot \rangle^\beta \cdot \mathbf{1} \{ p > 0 \} \cdot \log p ) \|_r < C \quad \forall p \in \mathcal{P}_D,
    \end{align*}
    where $\| \cdot \|_r$ is the usual $L^r$-norm.

  \item \label{ass:moment}
    Moment condition for densities: For all DAGs $D$ we have
    \begin{align*}
      \exists \gamma > s - d/r \quad \mathrm{s.t.} \quad \| p \cdot \langle \cdot \rangle^{\gamma-\beta} \|_r < \infty \quad \forall p \in \mathcal{P}_D,
    \end{align*}
    where $r, s, d,$ and $\beta$ are determined by~\ref{ass:sobolev}.

  \item \label{ass:variance}
    Uniformly bounded variance of log-densities: For all DAGs $D$ we have
    \begin{align*}
      \forall p^0 \in \mathcal{P}_D ~ \exists K > 0 \quad \mathrm{s.t.} \quad \sup_{p \in \mathcal{P}_D} var_{p^0} ( \log p(X_1, \ldots, X_d)) < K.
    \end{align*}

  \item \label{ass:closedness}
    Closedness of density classes: For all DAGs $D$ the induced density class $\mathcal{P}_D$ is a closed set, with the topology given by the Kullback-Leibler (KL) divergence $D_\mathrm{KL} (p(\mathbf{x}) || q(\mathbf{x})) = \int p(\mathbf{x}) \log \frac{p(\mathbf{x})}{q(\mathbf{x})} \iend{\mathbf{x}}$.

\end{enumerate}

The first two assumptions concern the general model setup and ensure identifiability (i.e.\ non-overlapping induced density classes). \ref{ass:anm}~requires the data to come from an identifiable ANM due to nonlinearity or non-Gaussianity, as in \citet{HoyPA09}. \ref{ass:minimality}~ensures there are no ``superfluous'' edges in the true DAG, i.e.\ the true model is the most parsimonious fitting the data.

The last four assumptions are technical and used to prove consistency of the penalized maximum likelihood estimator. \ref{ass:sobolev}~essentially requires the log-densities to be smooth. \ref{ass:moment}~requires the densities to have some (at least fractional) finite moments. \ref{ass:variance}~requires the log-densities, for every underlying density $p^0$, to have uniformly bounded second moments. Finally, \ref{ass:closedness}~guarantees the existence of the maximizers of the likelihood and the negative information entropy in each class. Furthermore, it is needed to ensure the true density $p^0$ has positive KL distance from all wrong density classes. Note that the latter statement alone would suffice to show consistency, since all statements can be written in terms of the supremums of likelihood and negative entropy, instead of their actual maximizers. However, for better comprehensability we chose the present formulation with the slightly stronger assumption.

Making these assumptions, the penalized maximum likelihood estimator is consistent. We show this by proving that the probability of the true model obtaining a smaller score than any other model vanishes in the large sample limit.

\begin{theorem} \label{thm:plconsistency}
  Assume \ref{ass:anm}--\ref{ass:closedness}. Let $S^i_n$ be the penalized likelihood score of DAG $D_i$, given by
  \begin{align*}
    S^i_n = \frac{1}{n} \sum_{j=1}^n \log \hat{p}_n^i (x_1^j, \ldots, x_d^j) - \# (\mathrm{edges})_i \cdot a_n,
  \end{align*}
  where $\# (\mathrm{edges})_i$ is the number of edges in DAG $D_i$, and $a_n = 1 / \log n$. Denote by $i_0$ the index of the true DAG $D_0$ = $D_{i_0}$. Then we have
  \begin{align*}
    P \left( S^{i_0}_n \le S^i_n \right) \rightarrow 0 \quad \mathrm{as} \quad n \rightarrow \infty \qquad \forall i \neq i_0.
  \end{align*}
\end{theorem}

The proof relies on entropy methods and is presented in the appendix. In practice the $1/\log n$ penalty rate might be too large. We used $a_n=1/\sqrt{n}$ for some simulations in Section~\ref{chapter:numerical} (where the noise is Gaussian), which lead to reasonably good performance for finite sample size $n=300$. Moreover, under stronger assumptions we have:

\begin{remark} \label{remark:rate}
  When replacing \ref{ass:variance} with the stronger assumption of sub-Exponential tails of $\log p(X_1, \ldots, X_d)$, we can improve the penalty rate $a_n$ in Theorem~\ref{thm:plconsistency} from $1/ \log n$ to $c n^{-1/(2+d/s)}$, for some $c>0$ sufficiently large.
\end{remark}

\section{Numerical Results} \label{chapter:numerical}
In this section we present simulation results to show that our method works under controlled conditions. In each case, the data generating process is an additive noise model with acyclic graph structure. We first reproduce some results from an earlier paper by \citet{HoyPA09}, where the model involves just two variables and is parametrized by two parameters, controlling linearity and Gaussianity respectively. Then, we extend this setup to a slightly more general class of models. Finally, we look at cases with more than two variables.

In our implementation we use (generalized) additive model regression (GAM, see \citet{HasTT86}) or local polynomial regression (LOESS, see \citet{CleW79}) for step~1 and logspline density estimation (see \cite{KooCS91}) or kernel density estimation for step~2. For models with more than two variables, penalization becomes important. We used a factor of $a_n = 1 / \sqrt{n}$ instead of the very severe $1/ \log n$. This can be justified since in the relevant simulations the noise is Gaussian and the log-densities can be assumed to be sub-Exponential. In this case, the faster rate can be used (see Remark~\ref{remark:rate}). All computations were carried out in the statistical computing language \texttt{R} (using packages \texttt{mgcv} and \texttt{logspline}) and the code is available on request from the authors.

\subsection{Identifiability depending on Linearity and Gaussianity}
\label{sec:hoyer}
\citet{HoyPA09} illustrate their method with a two-dimensional ANM of the form
\begin{align*}
  X_1 &= \epsilon_1 \\
  X_2 &= X_1 + bX_1^3 + \epsilon_2
\end{align*}
with the parameter $b$ ranging from $-1$ to $1$, thus controlling the linearity of the model. The noise terms $\epsilon_1, \epsilon_2$ are transformed Normal random variables:
\begin{align*}
  \epsilon_k = \mathrm{sgn}(\nu_k) \cdot |\nu_k|^q, \quad \nu_k \stackrel{\mathrm{iid}}{\sim} \mathcal{N}(0,1),
\end{align*}
where the parameter $q$ ranges from $0.5$ to $2$ and thus controls Gaussianity. The true direction $M_1: X_1 \rightarrow X_2$ cannot be identified with traditional methods (e.g.\ the PC algorithm), since the backwards model $M_2: X_1 \leftarrow X_2$ entails precisely the same conditional independence relations (none) and thus belongs to the same Markov equivalence class. If $b=0$ and $q=1$ there exists a backwards model entailing the same joint density. As soon as we move away from this point, however, the model becomes identifiable \citep{HoyPA09}. We confirm this numerically, showing our method performs as expected in this setting.

We discretize the parameter space $(b,q) \in [-1,1] \times [0.5,2]$, and for each grid point we repeat the simulation 1000 times, with $n=300$ samples per trial. We then count the number of times the backwards model gets wrongly chosen by the method\footnote{I.e. when the likelihood score of the backwards model is lower than that of the forwards model.}, and this false decision rate serves as our measure of quality of the method. As can be seen in Figure \ref{fig:bqfull}, the false decision rate peaks around $(b,q) = (0,1)$ with around 50\% wrong decisions, corresponding to random guessing. Away from this region it quickly drops to zero. In this setting the regressions were done using LOESS and the density estimations using logsplines.

\begin{figure}
  \centering
  \begin{subfigure}{0.3\textwidth}
    \centering
    \includegraphics[width=\textwidth]{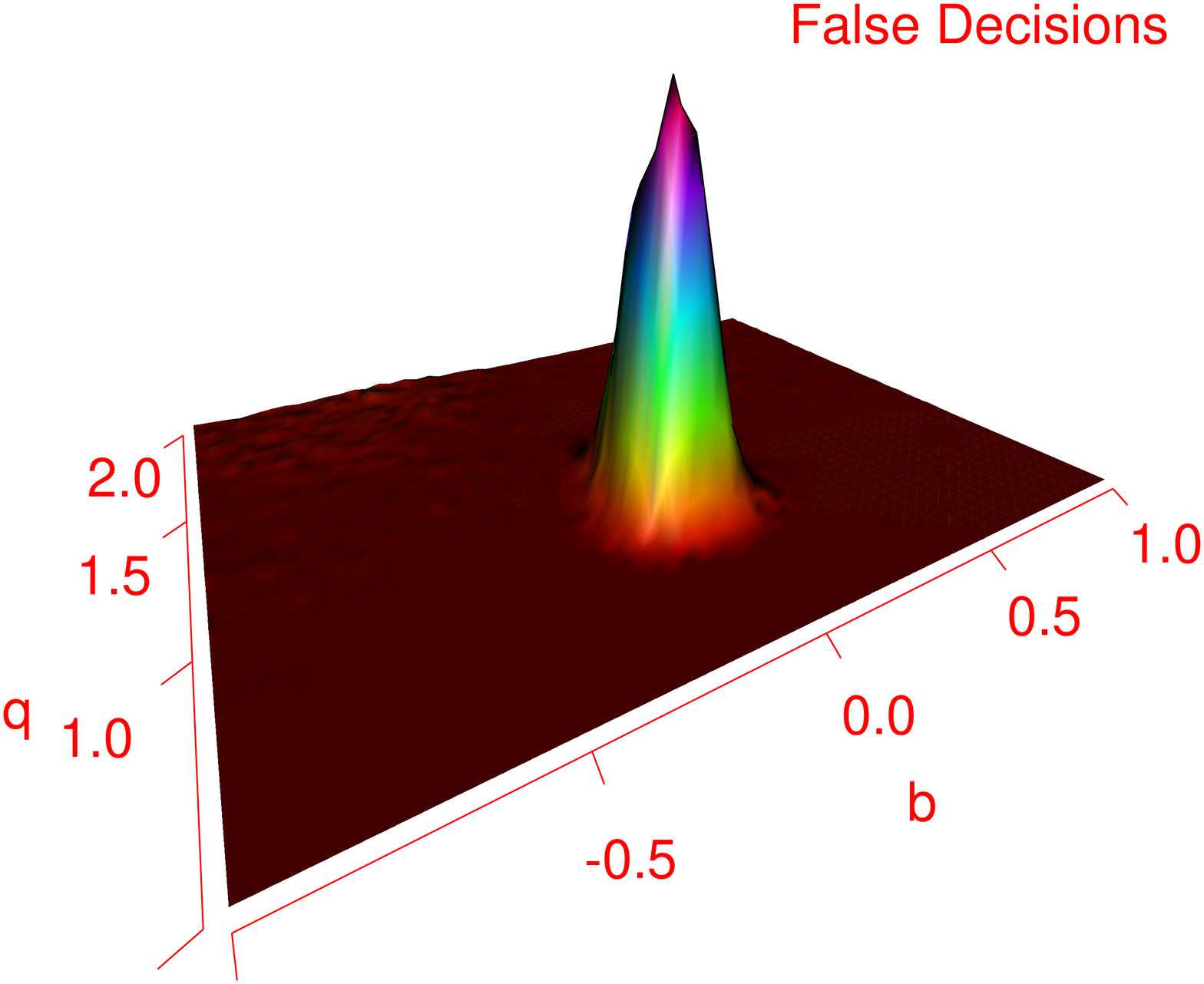}
    \caption{Full $b \times q$ grid.}
  \end{subfigure}
  \begin{subfigure}{0.3\textwidth}
    \centering
    \includegraphics[width=\textwidth]{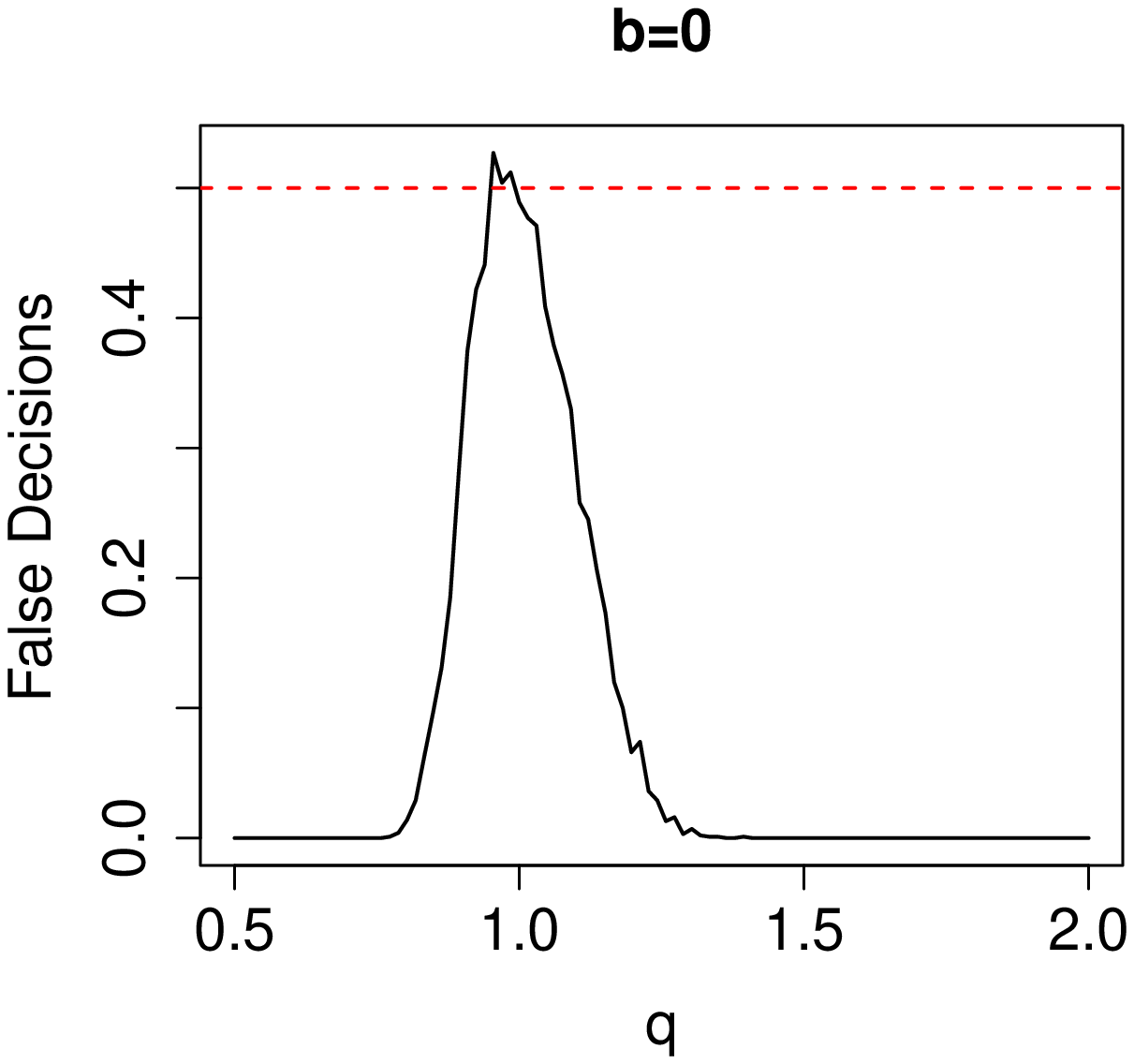}
    \caption{$b$ fixed.}
  \end{subfigure}
  \begin{subfigure}{0.3\textwidth}
    \centering
    \includegraphics[width=\textwidth]{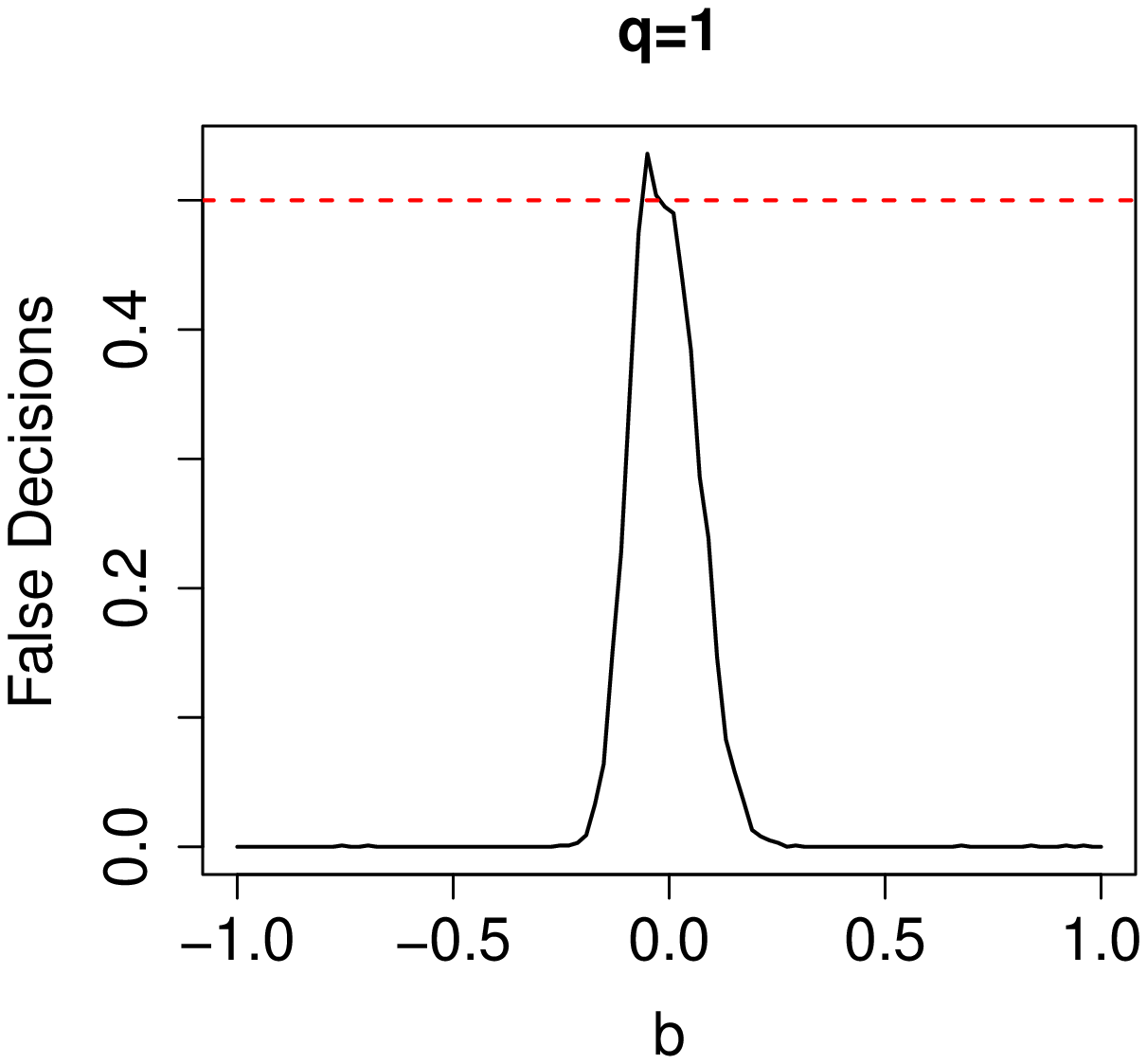}
    \caption{$q$ fixed.}
  \end{subfigure}
  \caption{False decision rates for a two-dimensional ANM with two parameters $b$ and $q$, controlling linearity and Gaussianity ($n=300$). For $b=0$ the model is linear, for $q=1$ the noise is Gaussian.}
  \label{fig:bqfull}
\end{figure}

\subsection{Random Edge Functions}
We now generalise the setup of the scenario from Section~\ref{sec:hoyer} in allowing a bigger function class for the edge function. Specifically, we randomly generate functions by sampling a random path from a Wiener process and smoothing it with cubic splines\footnote{A Wiener path (random normal increments) is sampled on a 1000 point grid spanning $[-1,1]$ and the resulting vector rescaled to an interval of length 2 and consequently smoothed using cubic splines. The resulting functions are linear outside [-1,1] and nonlinear inside.}. To measure their nonlinearity we use the normalised $L^2$-difference between the function and its best linear approximation on the interval $[-1,1]$, as described in \citet{EmaKK93}. A number of randomly generated functions with different nonlinearity values are shown in Figure~\ref{fig:remarginal3}. We again choose a uniform grid of nonlinearity values (in the interval $[0, 0.4]$) and, for each grid point, generate 100 random functions. With each function we perform 100 simulations and average the results. The noise is standard Gaussian in this setting. In Figure~\ref{fig:remarginal3} we see the results for a small sample ($n=300$) and a large sample ($n=1500$) case. The findings are analogous to the simple cubic model---the false decision rate decreases with nonlinearity of the edge function and sample size. Again, the regressions were done using LOESS and the density estimations using logsplines.

\begin{figure}
  \centering
  \begin{subfigure}{0.45\textwidth}
    \centering
    \includegraphics[width=\textwidth]{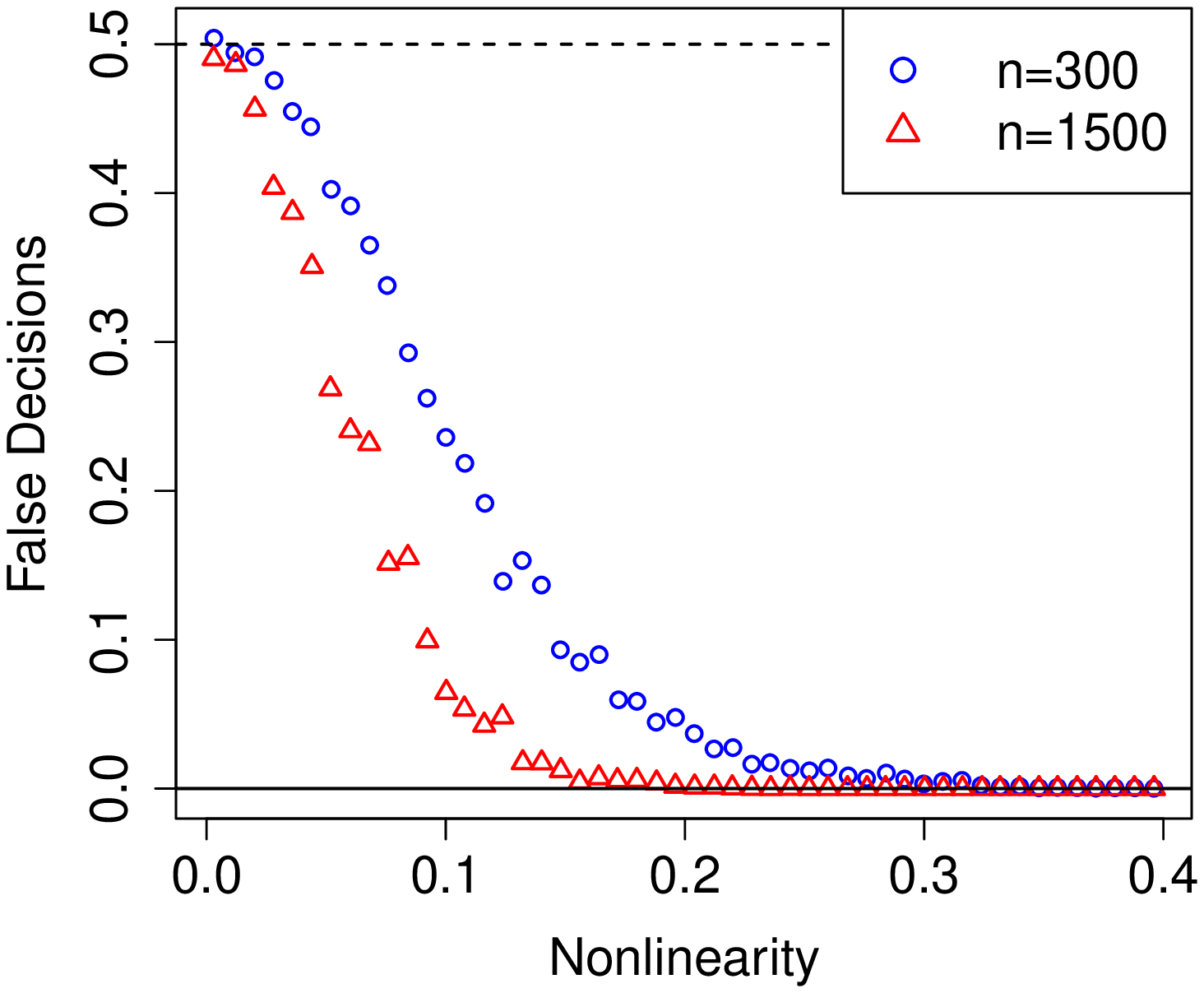}
    \caption{}
  \end{subfigure}
  \begin{subfigure}{0.45\textwidth}
    \centering
    \includegraphics[width=\textwidth]{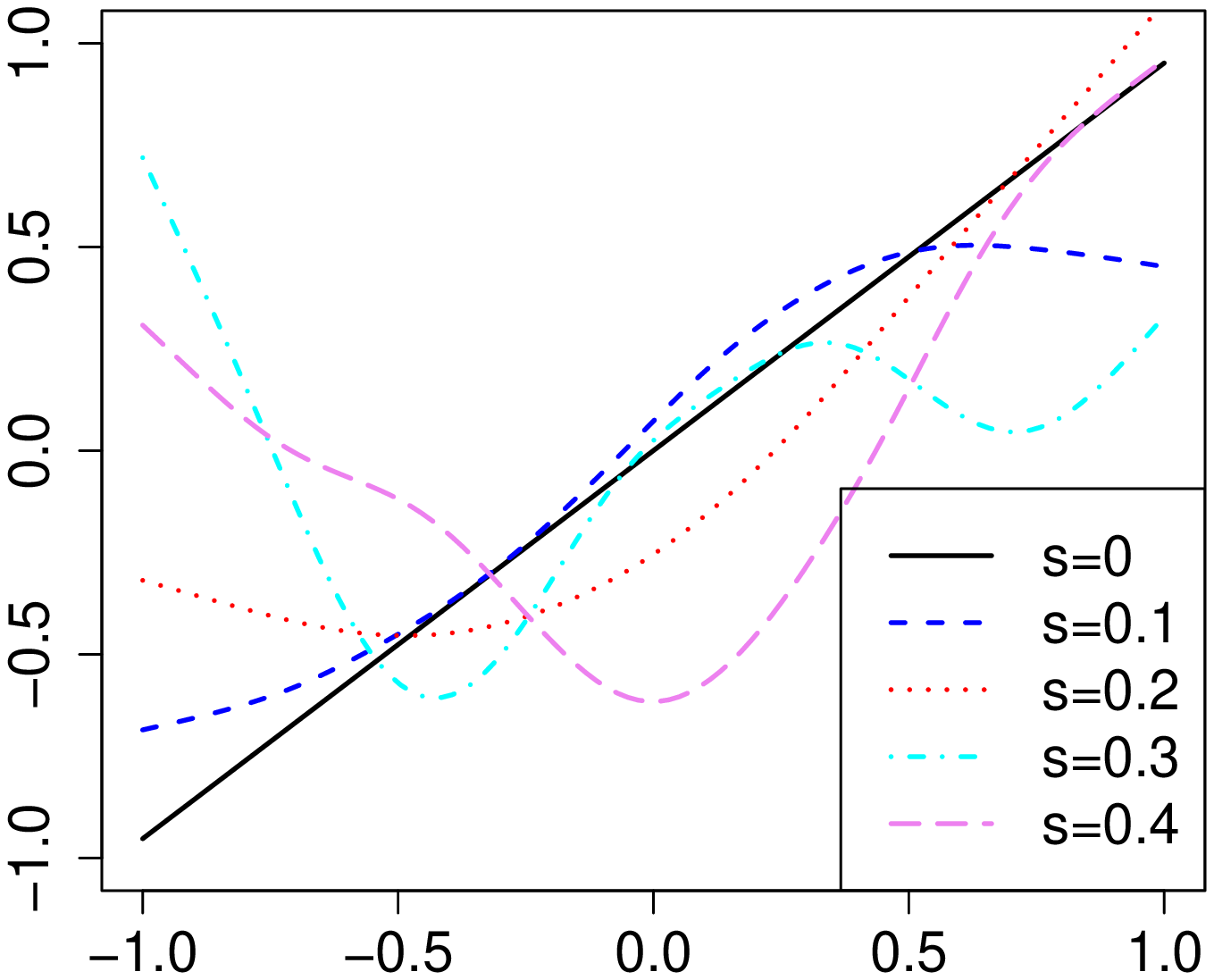}
    \caption{}
  \end{subfigure}
  \caption{a) False decision rates with randomly sampled edge functions and Gaussian noise decreases with nonlinearity of the functions. b) Examples of randomly generated functions, where parameter $s$ controls nonlinearity.}
  \label{fig:remarginal3}
\end{figure}

\subsection{Larger Networks and Thresholding}
\label{sec:thresh}
In a practical situation the reliability of any method invariably depends on whether its assumptions are met, as well as some other factors. In our case this would include the nonlinearity of the edge functions, the non-Gaussianity of the noise, the sample size, and the number of nodes. It would be desirable to have some criterion indicating there is insufficient information to make a decision. While this is hard to make concrete, a good first heuristic seems to be the separation of the best-scoring model from the rest. We concretely look at the ratio of the smallest ($\Delta_1$) and the largest ($\Delta_2$) score difference (see Figure~\ref{fig:thresh}). If this is smaller than some threshold $t$, we make no decision (no selection of a model).

The effect of this can be seen in Figure~\ref{fig:shdbic}. Starting from a full DAG with 3 nodes as the ground truth, we randomly generate 100 different sets of nonlinear\footnote{With nonlinearity values in $[0.39,0.4]$.} edge functions, and for each set of edge functions we generate 100 data sets with standard Gaussian noise of sample size $n=300$. With each data set we run an exhaustive search over all 25 candidate models and, if making a decision after thresholding, compute the structural Hamming distance (SHD) between the best-scoring DAG and the ground truth. Comparing the thresholds $t=0$ and $t=0.01$, the false decision rate falls from 3.9\% to 2.4\% while in 3.1\% of the cases no decision is made.

We also look at two simulation settings suggested in \citet{PetJ11}, where the graph consists of 4 nodes and the edge functions are nonlinear but parametrized by 4 and 5 parameters respectively. In both cases, \texttt{nonlinear1} and \texttt{nonlinear2}, 100 sets of parameters are drawn from a uniform distribution and then data (with a sample size of $n=400$) is generated. Our method identifies the correct DAG in 96 / 97 out of the 100 cases for \texttt{nonlinear1/2} (in the other cases, there is one additional edge). This certainly improves upon the results reported in \citet{PetJ11} (86 correct decision in both cases).

In all of these multivariate settings, we used GAM for regression and logsplines for density estimation.

\begin{figure}
  \centering
  \begin{subfigure}{0.45\textwidth}
    \centering
    \includegraphics[width=\textwidth]{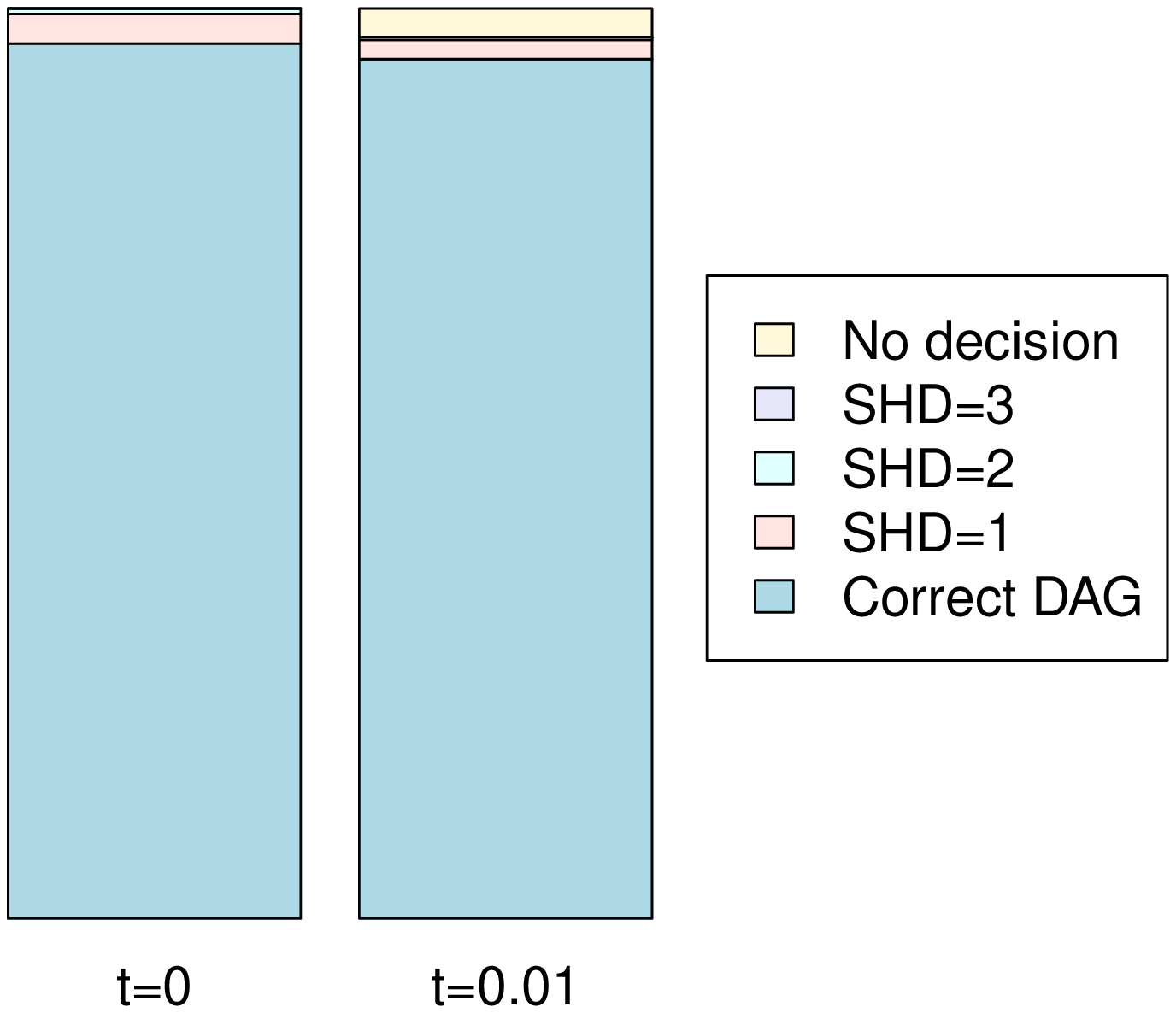}
    \caption{}
  \label{fig:shdbic}
  \end{subfigure}
  \begin{subfigure}{0.45\textwidth}
    \centering
    \includegraphics[width=\textwidth]{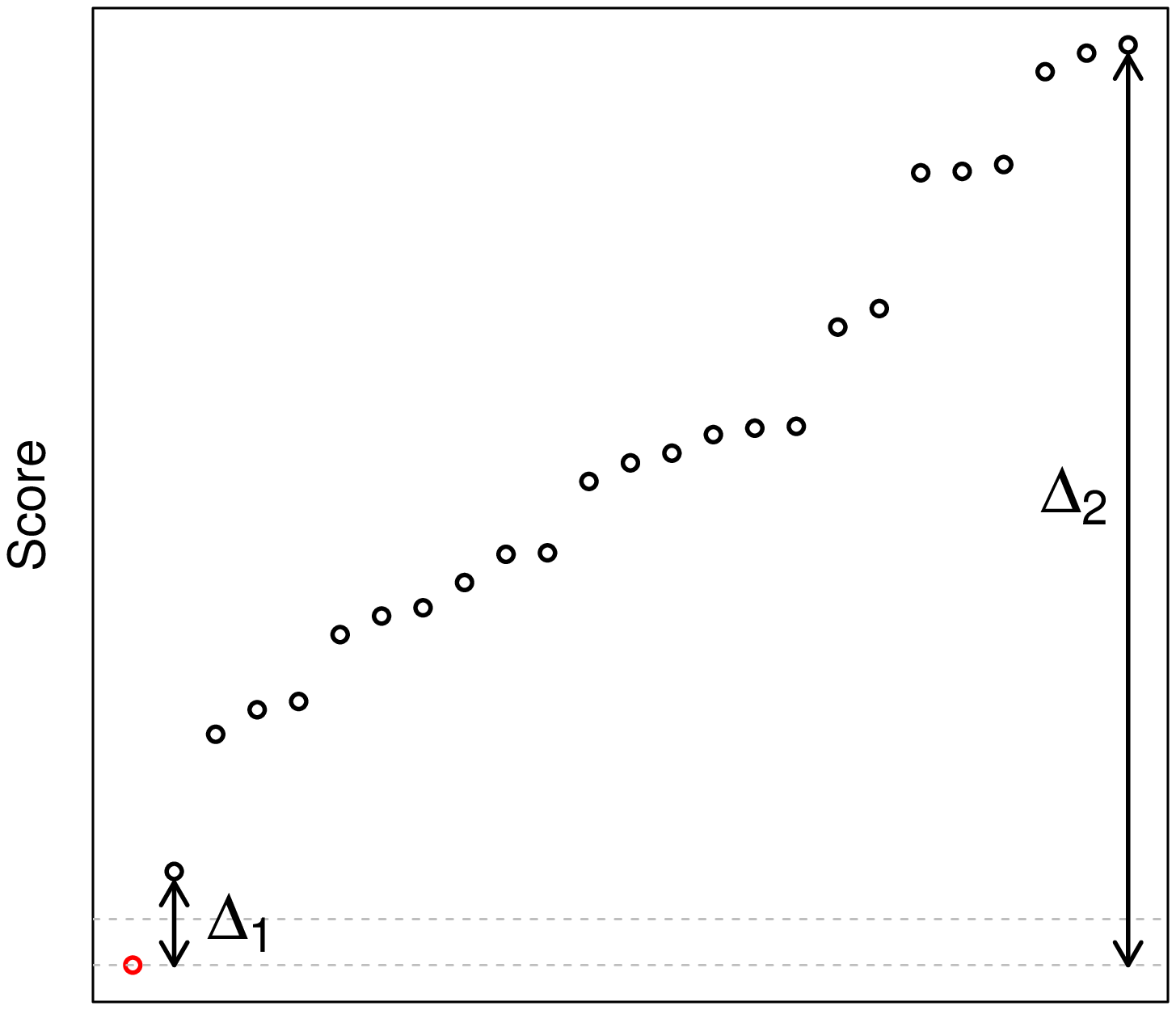}
    \caption{}
    \label{fig:thresh}
  \end{subfigure}
  \caption{a) Structural Hamming distance between the best-scoring DAG and the ground truth for a 3-node simulation with ($t=0.01$) and without ($t=0$) thresholding. b) Illustration of thresholding for a single simulation run. Let $s_1, \ldots, s_D$ be the (increasingly) ordered scores. Then $\Delta_1=s_1/s_2$ and $\Delta_2=s_1/s_N$.}
\end{figure}

\section{Real Data}
To determine the performance on real-world datasets, we apply our method to so-called cause-effect pairs. These are bivariate datasets where the true causal direction is known. An example would be the altitude and the average temperature of weather stations. \citet{MooJJ10} describe 8 such pairs and compare several methods that were submitted as part of the Causality Pot-Luck Challenge. Our method identifies 7 out of the 8 pairs correctly\footnote{This corresponds to a p-value of 0.0352 under the random guessing null hypothesis.}, thus beating all other compared methods except \citet{ZhaKH10}, who take into account post-nonlinear additive noise.

We next consider the extended collection of cause-effect pairs, which can be found at \url{http://webdav.tuebingen.mpg.de/cause-effect}. This currently comprises 86 datasets, 81 of which are bivariate. Using our method on these 81 bivariate datasets, we identify the true model in 66\% of the cases\footnote{This corresponds to a p-value of 0.005 under the random guessing null hypothesis.}. In \citet{JanDA12} a subset of these datasets were used to compare various causal inference methods. Running our method on those datasets, it compares well with the other methods (see Table~\ref{table:causeeffectpairs}), being slightly better than independence testing (AN) and outperforming the Lingam method.

In both of these settings we used LOESS and kernel density estimation.


\begin{table}
  \centering
  \begin{tabular}{l | c | c | c | c | c | c}
    Method & SCL & AN & Lingam & PNL & IGCI & GPI \\ \hline
    Accuracy & 66\% & 63\% & 58\% & 68\% & 75\% & 70\%
  \end{tabular}
  \caption{Success rates of different causal inference methods on cause-effect pairs at a decision rate of 100\%. SCL=Score-based Causal Learning (our method), AN=Additive Noise with independence testing, PNL=Post-Nonlinear, IGCI=Information-Geometric Causal Inference, GPI=Gaussian Process Inference. All values except SCL taken from \citet{JanDA12}. All datasets were subsampled three times (if $n>500$), and the results were averaged.}
  \label{table:causeeffectpairs}
\end{table}

\section{Conclusions}
We presented a new fully nonparametric likelihood score-based method for causal inference in nonlinear or non-Gaussian ANMs. We proved consistency of the penalized maximum likelihood estimator for finding the correct model. We showed via simulation studies that our method works well in practice when the ground truth is an ANM with sufficiently nonlinear edge functions or non-Gaussian error terms. Our method compares favourably to other causal inference procedures on both simulated and real-world data.

As a major open challenge, the current approach of exhaustively searching through the whole model space becomes computationally infeasible for more than a handful of variables. Since our method is score-based and the scoring criterion is local (i.e., decomposable), it is straightforward to implement a greedy algorithm although there will be no guarantee for finding a global optimum.

\appendix
\section{Consistency Proof}
The proof heavily relies on entropy methods and empirical process theory. For a good overview of the necessary material we refer to \citet{vdGeeS00} or \citet{vdVaa96}. For an overview of Sobolov and related function spaces we refer to \citet{TriH83}.

Throughout this section we will adopt the following notation for taking expectations of some random variable $f$ with respect to a distribution $Q$ (following~\citet{vdGeeS00}):
\begin{align*}
  Q f := \int f \iend{Q}.
\end{align*}
In particular, this means we will write expectations and means as
\begin{align*}
  P f &= \mathbb{E} \left[ f(X) \right] \\
  P_n f &= \frac{1}{n} \sum_{j=1}^n f(X^j),
\end{align*}
where $P$ is the true distribution with density $p^0$, $f:\mathbb{R}^d \rightarrow \mathbb{R}$ is some function, $X$ is a vector of random variables (one corresponding to each node) with distribution $P$, $\{ X^j \}_{j=1,\ldots,n}$ are independent copies of $X$, and $P_n$ is the empirical distribution (placing weight $1/n$ on each $X^j$).

With this notation we can write the maximum likelihood estimator $\hat{p}_n^i$ and the entropy minimizer $p^i$ in class $\mathcal{P}_i$ (which exist by assumption~\ref{ass:closedness} but need not be unique) as:
\begin{align}
  \hat{p}_n^i &= \argmax_{p \in \mathcal{P}^i} P_n \log p, \label{def:mle2} \\
  p^i &= \argmax_{p \in \mathcal{P}^i} P \log p. \label{def:entropy2}
\end{align}
Note that the true density $p^0$ minimizes the information entropy over the complete density space $\bigcup_{i=1}^N \mathcal{P}^i$ since the Kullback-Leiber divergence $P \log \frac{p^0}{p}$ is positive for all densities $p \ne p^0$.

One of the building blocks of the proof of Theorem~\ref{thm:plconsistency} is a uniform law of large numbers (ULLN) for the classes of log-densities:
\begin{align*}
  \sup_{p \in \mathcal{P}^i} | (P_n - P) \log p | \stackrel{P}{\longrightarrow} 0 \quad \mathrm{as} \quad n \rightarrow \infty \qquad \forall i.
\end{align*}
To show this, an entropy argument is used. We first define the bracketing entropy of a function space. Let $\mathcal{G}$ be a set of functions from $\mathbb{R}^d$ to $\mathbb{R}$. Two functions $g^L,g^U:\mathbb{R}^d \rightarrow \mathbb{R}$ (not necessarily in $\mathcal{G}$) form an \emph{$\epsilon$-bracket} for some $g \in \mathcal{G}$, if $g^L \le g \le g^U$ and $\|g^L - g^U\|_{1,\mu} < \epsilon$, where $\|\cdot\|_{1,\mu}$ is the weighted $L^1$-norm, i.e.\ $\|f\|_{1,\mu} = \int | f(x) \mu(x) | \iend{x}$. Suppose $\{g^L_i,g^U_i\}_{i=1,\ldots,N_{[]}}$ is the smallest set s.t.\ $\forall g \in \mathcal{G}$ $\exists i$ s.t.\ $g^L_i,g^U_i$ form an $\epsilon$-bracket for $g$, where $N_{[]}$ denotes the number of such pairs. Then $H_{[]}(\epsilon, \mathcal{G}, \|\cdot\|_{1,\mu}) := \log N_{[]}$ is called the bracketing entropy of $\mathcal{G}$.

The following result connects bracketing entropy \mbox{$H_{[]} (\epsilon, \mathcal{G}, \|\cdot\|_{1,p^0})$} with respect to the $L^1$-norm weighted with the true density $p^0$ and the uniform convergence of the empirical process $(P_n-P)g$. Note that here and throughout this section we use the notation "$a(\epsilon) \lesssim b(\epsilon)$" as shorthand for "$a(\epsilon) \le c b(\epsilon)$ $\forall \epsilon > 0$ for some constant $c$ not depending on $\epsilon$".

\begin{lemma}
  \label{thm:ulln}
  Suppose that:
  \begin{enumerate}
    \renewcommand{\theenumi}{(\roman{enumi})}
    \renewcommand{\labelenumi}{\theenumi}
    \item
    $\exists ~0 \le \alpha < 1$ s.t.\ $H_{[]}(\epsilon, \mathcal{G}, \|\cdot\|_{1,p^0}) \lesssim \epsilon^{-\alpha} \quad \forall \epsilon > 0$ and
    \item \label{cond:variance}
    $\exists K$ s.t.\ $var \left( g (X_1, \ldots, X_d) \right) < K \quad \forall g \in \mathcal{G}$
  \end{enumerate}
  Then $\mathcal{G}$ satisfies the ULLN:
  \begin{align*}
    \mathbb{P} \left( \sup_{g \in \mathcal{G}} |(P_n - P) g | > \delta_n \right) \rightarrow 0 \qquad \mathrm{as} \qquad n \rightarrow \infty,
  \end{align*}
  where $\delta_n = c / \log n$ for some $c > 0$.
\end{lemma}
\begin{proof}
  We first show that it suffices to look at the supremum over the bracketing functions. Let $g \in \mathcal{G}$ and $g^L_i, g^U_i$ be its $\delta_n$-brackets. We then have
  \begin{align*}
    \left( P_n - P \right) g &< \left( P_n - P \right) g^U_i + \delta_n \\
    \mathrm{and} \quad &> \left( P_n - P \right) g^L_i - \delta_n.
  \end{align*}
  So we have
  \begin{align*}
    \left| \left( P_n - P \right) g \right| < \max_{i = 1, \ldots, N_{[]}} \left( \left| \left( P_n - P \right) g^L_i \right|, \left| \left( P_n - P \right) g^U_i \right| \right) + \delta_n
  \end{align*}
  and hence
  \begin{align*}
    \sup_{g \in \mathcal{G}} \left| \left( P_n - P \right) g \right| < \max_{g \in \{g^L_i,g^U_i\}_i} \left| \left( P_n - P \right) g \right| + \delta_n.
  \end{align*}
  Now
  \begin{align}
    \mathbb{P} \left( \sup_{g \in \mathcal{G}} | ( P_n - P ) g | > 2 \delta_n \right) & \le \mathbb{P} \left( \max_{g \in \{g^L_i,g^U_i\}_i} | ( P_n - P ) g | > \delta_n \right) \notag\\
    & \le 2 N_{[]} (\delta_n) \max_{g \in \{g^L_i,g^U_i\}_i} \mathbb{P} \left( \left| \left( P_n - P \right) g \right| > \delta_n \right) \notag\\
    & \lesssim \exp (\delta_n^{-\alpha}) \frac{K^2}{n \delta_n^2} \label{eqn:chebyshev}
  \end{align}
 where the last line follows from Chebyshev's inequality. Substituting for $\delta_n$ gives
  \begin{align*}
    \mathbb{P} (\ldots) \lesssim \log^2 n \cdot \exp (c^{-\alpha} \log^{\alpha} n - \log n ) \longrightarrow 0 \quad \mathrm{as} \quad n \rightarrow \infty.
  \end{align*}
\end{proof}

Note that if we replace condition~\ref{cond:variance} with the assumption that $g(X_1, \ldots, X_d)$ are sub-Exponential (as in Remark~\ref{remark:rate}), we apply the sub-Exponential tail bound (see \citet[Lemma 14.9]{BuePV11} for example) instead of Chebyshev's inequality and obtain $\exp (\delta_n^{-\alpha} - \frac{n \delta_n^2}{const.})$ instead of~\eqref{eqn:chebyshev}, which converges to zero for $\delta_n = c n^{-1/(2+\alpha)}$, for $c>0$ sufficiently large.

Lemma~\ref{thm:ulln} shows that a sufficient condition for the ULLN is finite bracketing entropy. To this end, we make use of the following result:

\begin{lemma}[{\citet[Theorem 1]{NicRP07}}] \label{thm:entropy}
  Suppose $\mathcal{G}$ is a (non-empty) bounded subset of the weighted Sobolev space $W_p^s (\mathbb{R}^d,\langle x \rangle^\beta)$ for some $\beta < 0$. Suppose $\exists \gamma > s - d/p > 0$ s.t.\ the moment condition
  \begin{align*}
    \|\langle \cdot \rangle^{\gamma-\beta} \|_{1,\mu} = \| \mu(x) \langle x \rangle^{\gamma-\beta} \|_1 < \infty
  \end{align*}
  holds for some Borel measure $\mu$ on $\mathbb{R}^d$. Then:
  \begin{align*}
    H_{[]} (\epsilon, \mathcal{G}, \| \cdot \|_{1,\mu} ) \lesssim \epsilon^{-d/s}.
  \end{align*}
\end{lemma}

The relevant sets of functions $\mathcal{G}$ in this context are the log-densities of each class, i.e.\ $\{ \mathbf{1} \{ p > 0 \} \log p ~|~ p \in \mathcal{P}^i \}$, with the relevant Borel measure $\mu$ being the true density $p^0$.

Essentially the idea of the proof of Theorem~\ref{thm:plconsistency} is to show that the maximum log-likelihood in each induced density class converges to the minimal entropy. For non-overlapping models (e.g.\ $X_1 \rightarrow X_2$ and $X_1 \leftarrow X_2$), the minimal entropy will be different in each class (with the minimum occuring in the true model class), and the likelihood will eventually pick up on this difference. Since the penalty term vanishes asymptotically, an ever so small difference in entropy will differentiate the true model class from the others. For overlapping (e.g. hierarchical) models, the minimal entropy can occur in more than one class. In this case the penalty term picks out the most parsimonious model (which is the true model according to the Causal Minimality assumption). Note that the penalty $1/\log n$ is quite large compared with e.g.\ the BIC penalty ($\log n/n$). This is due to the slow convergence of maximum likelihood to minimal entropy (Lemmas~\ref{thm:entropyconv} and \ref{thm:ulln}). If the penalty vanishes too quickly, it will be drowned out by the noise in the likelihood and have no effect. The convergence can be improved (and thus the penalty relaxed) when making stronger assumptions on the distributions, e.g.\ sub-Gaussian tails.

The following lemma shows convergence of maximum log-likelihood to minimal entropy in each class, given that a ULLN holds.

\begin{lemma} \label{thm:entropyconv}
  Suppose that a ULLN for the classes $\log \mathcal{P}^i$ holds with convergence rate $\delta_n$, i.e.\
  \begin{align*}
    P \left( \sup_{p \in \mathcal{P}^i} \left| \left( P_n - P \right) \left( \mathbf{1} \{ p > 0 \} \log p \right) \right| > \delta_n \right) \rightarrow 0 \quad \mathrm{as} \quad n \rightarrow \infty.
  \end{align*}
  Then
  \begin{align*}
    P \left( \left| P_n \log \hat{p}^i_n - P \log p^i \right| > \delta_n \right) \rightarrow 0 \quad \mathrm{as} \quad n \rightarrow \infty.
  \end{align*}
\end{lemma}
\begin{proof}
  By the definition of the MLE \eqref{def:mle2} we have:
  \begin{align*}
    P_n \log \hat{p}^i_n \ge P_n \log p^i &= P \log p^i + (P_n - P) \log p^i,
  \end{align*}
  i.e.
  \begin{align}
    \label{eqn:lowerbound}
    P_n \log \hat{p}^i_n - P \log p^i &\ge (P_n - P) \log p^i.
  \end{align}
  Let $\tilde{\mathcal{P}}^i_n$ be the restriction of $\mathcal{P}^i$ to densities whose support contains the data, i.e.\ $\tilde{\mathcal{P}}^i_n = \{ p \in \mathcal{P}^i ~|~ \mathrm{supp} (p) \supseteq \{ X^1, \ldots, X^n \} \}$. Note that the maximum log-likelihood as well as minimum entropy are the same over $\mathcal{P}^i$ and $\tilde{\mathcal{P}}^i_n$, since densities with support not including the data will yield values of $-\infty$. So we also have:
  \begin{align*}
    P_n \log \hat{p}^i_n &= \max_{p \in \mathcal{P}^i} P_n \log p = \max_{p \in \tilde{\mathcal{P}}^i_n} P_n \log p \\
    &= \max_{p \in \tilde{\mathcal{P}}^i_n} \left( P \log p + (P_n - P) \log p \right) \\
    & \le P \log p^i + \sup_{p \in \tilde{\mathcal{P}}^i_n} (P_n - P) \log p,
  \end{align*}
  i.e.
  \begin{align*}
    P_n \log \hat{p}^i_n - P \log p^i &\le \sup_{p \in \tilde{\mathcal{P}}^i_n} (P_n - P) \log p.
  \end{align*}
  This together with \eqref{eqn:lowerbound} yields:
  \begin{align*}
    \left| P_n \log \hat{p}^i_n - P \log p^i \right| & \le \max \left( \left| (P_n - P) \log p^i \right|, \sup_{p \in \tilde{\mathcal{P}}^i_n} (P_n - P) \log p \right) \\
    & \le \max \left( \left| (P_n - P) \log p^i \right|, \sup_{p \in \tilde{\mathcal{P}}^i_n} \left| (P_n - P) \right| \log p \right) \\
    & \le \sup_{p \in \tilde{\mathcal{P}}^i} \left| (P_n - P) \log p \right| \\
    & \le \sup_{p \in \mathcal{P}^i} \left| (P_n - P) \left( \mathbf{1} \{ p > 0 \} \log p \right) \right|.
  \end{align*}
  We thus have:
  \begin{align*}
    & P \left( \left| P_n \log \hat{p}^i_n - P \log p^i \right| > \delta_n \right) \\
    & \qquad \qquad \qquad \qquad \le P \left( \sup_{p \in \mathcal{P}^i} \left| \left( P_n - P \right) \left( \mathbf{1} \{ p > 0 \} \log p \right) \right| > \delta_n \right),
  \end{align*}
  which converges to zero as $n \rightarrow \infty$ by assumption.
\end{proof}

Finally, before proving Theorem~\ref{thm:plconsistency}, we show the following useful lemma.

\begin{lemma} \label{thm:ineq}
  Let $a, b, a', b' \in \mathbb{R}$ and $\epsilon > 0$. If one of the following holds:
  \begin{enumerate}
    \item $a - b > \epsilon$ and $a' - b' \le 0$
    \item $a - b < \epsilon$ and $a' - b' \ge 2 \epsilon$
  \end{enumerate}
  we have $|a-a'| > \frac{\epsilon}{2}$ or $|b-b'| > \frac{\epsilon}{2}$.
\end{lemma}
\begin{proof}
  Assume (i). Then we have
  \begin{align*}
    \epsilon = \epsilon - 0 \le a - b + b' - a' = |a-a'-(b-b')| \le |a-a'| + |b-b'|,
  \end{align*}
  and the result follows. Similarly for (ii):
  \begin{align*}
    \epsilon = 2\epsilon - \epsilon \le a' - b' + b - a = |a'-a-(b'-b)| \le |a'-a| + |b'-b|.
  \end{align*}
\end{proof}

We can now prove the main theorem.

\begin{proof}[Proof of Theorem \ref{thm:plconsistency}]
  We will make repeated use of Lemma~\ref{thm:entropyconv}. For that matter, note that assumptions~\ref{ass:sobolev}, \ref{ass:moment}, and~\ref{ass:variance}, together with Lemmas~\ref{thm:ulln} and~\ref{thm:entropy} (taking $\mu = p^0$) satisfy the sufficient conditions. \ref{ass:closedness}~ensures the existence of $\hat{p}_n^i, p^i$ as defined in~\eqref{def:mle2} and~\eqref{def:entropy2}.

  Let $i \neq i_0$. We differentiate two cases: i) where $\mathcal{P}^i$ includes the true density $p^0$ and ii) where it does not. Let \mbox{$\delta_n = \left( \# (\mathrm{edges})_i - \# (\mathrm{edges})_{i_0} \right) \cdot \frac{1}{\log n}$} denote the difference of the penalties in the two scores.

  Case i). $p^0 \in \mathcal{P}^i$, which implies $p^i = p^0$. Assumptions~\ref{ass:anm} and \ref{ass:minimality} together with Theorem~2 in \citet{PetJ11} guarantee identifiability of the true graph. In particular this means that in this case $\mathcal{P}^i$ must correspond to a graph containing the true graph. Hence $\# (\mathrm{edges})_i > \# (\mathrm{edges})_{i_0}$, i.e.\ $\delta_n > 0$. We then have:
  \begin{align*}
    P \left( S^{i_0}_n \le S^i_n \right) & \le P \left( P_n \log \hat{p}^i_n - P_n \log \hat{p}^{i_0}_n > \frac{\delta_n}{2} \right) \\
    & \le P \left( \left| P_n \log \hat{p}^{i_0}_n - P \log p^0 \right| > \frac{\delta_n}{4} \quad \vee \right. \\
    & \qquad \qquad \qquad \qquad \qquad \left. \left| P_n \log \hat{p}^i_n - P \log p^i \right| > \frac{\delta_n}{4} \right) \\
    & \le P \left( \left| P_n \log \hat{p}^{i_0}_n - P \log p^0 \right| > \frac{\delta_n}{4} \right) + \\
    & \qquad \qquad \qquad \qquad \enspace \: P \left( \left| P_n \log \hat{p}^i_n - P \log p^i \right| > \frac{\delta_n}{4} \right) \rightarrow 0
  \end{align*}
  as $n \rightarrow \infty$, where the second line follows from $p^i=p^0$ and Lemma~\ref{thm:ineq} (first case), and the convergence in the last line follows from Lemma~\ref{thm:entropyconv}.

  Case ii). $p^0 \notin \mathcal{P}^i$, which implies $P \log p^0 > P \log p^i$. Hence $\exists \delta > 0$ s.t.\ $P \log p^0 > P \log p^i + 4 \delta$. Let $N > 0$ s.t.\ $\# (\mathrm{edges})_{i_0} \cdot \frac{1}{\log n} < \delta ~ \forall n \ge N$. Then we have
  \begin{align*}
    P \left( S^{i_0}_n \le S^i_n \right) &= P \left( P_n \log \hat{p}^{i_0}_n - P_n \log \hat{p}^i_n \le - \delta_n \right) \\
    & \le P \left( P_n \log \hat{p}^{i_0}_n - P_n \log \hat{p}^i_n < \delta \right) \\
    & \le P \left( \left| P_n \log \hat{p}^{i_0}_n - P \log p^0 \right| > \delta \quad \vee \right. \\
    & \qquad \qquad \qquad \qquad \qquad \left. \left| P_n \log \hat{p}^i_n - P \log p^i \right| > \delta \right) \\
    & \le P \left( \left| P_n \log \hat{p}^{i_0}_n - P \log p^0 \right| > \delta \right) + \\
    & \qquad \qquad \qquad \qquad \enspace \; P \left( \left| P_n \log \hat{p}^i_n - P \log p^i \right| > \delta \right) \rightarrow 0 \\
  \end{align*}
  as $n \rightarrow \infty$, where the third line follows from Lemma~\ref{thm:ineq} (second case), and the convergence in the last line follows again from Lemma~\ref{thm:entropyconv}.
\end{proof}

\bibliography{scl}

\end{document}